\documentclass[10pt,conference]{IEEEtran}
\usepackage{amsmath,amssymb,amsfonts,amsthm}
\newtheorem{proposition}{Proposition}
\usepackage{graphicx}
\usepackage{booktabs}
\usepackage{cite}
\usepackage{url}
\usepackage{algorithm}
\usepackage{algorithmic}
\usepackage{stfloats}
\usepackage{placeins}

\begin{document}

\title{Depth-Dependent Hidden-State Collapse in Dynamical System Autoencoders for LiDAR Point-Cloud Classification}

\author{
\IEEEauthorblockN{Patricia Medina}
\IEEEauthorblockA{
Department of Mathematics\\
New York City College of Technology, CUNY\\
Brooklyn, NY, USA\\
and\\
The Graduate Center, CUNY\\
New York, NY, USA\\
Email: Patricia.Medina36@citytech.cuny.edu
}
\and
\IEEEauthorblockN{Hy P. G. Lam}
\IEEEauthorblockA{
Department of Mathematical Sciences\\
Worcester Polytechnic Institute\\
Worcester, MA, USA\\
Email: hlam@wpi.edu
}
}

\maketitle

\begin{abstract}
We study Dynamical System Autoencoders (DSAE) for LiDAR point-cloud classification using spatial coordinates and Product Coefficient feature augmentations. The experiments compare separately trained DSAE architectures at encoder depths $K=1,\ldots,5$ and evaluate the resulting hidden representations with Random Forest, kNN, and a majority-class Dummy baseline. The main finding is a hidden-state collapse at $K=5$. For both xyz and xyz plus Product Coefficient inputs, the hidden-state standard deviation falls to the order of $10^{-5}$, while all three classifiers attain the same macro F1 score of $0.224688$. We prove that between-class hidden scatter is bounded by total hidden scatter, which in turn is controlled by the reported hidden-state variance. Thus a nearly constant hidden representation cannot retain substantial class-separating structure. Product Coefficients neither improve pre-collapse macro F1 nor prevent the $K=5$ collapse in the present DSAE setting. These results identify large-depth representation collapse as a concrete failure mode for DSAE LiDAR classification.

\end{abstract}

\section{Introduction}

LiDAR point clouds contain geometric information at several scales, but their dimensionality makes learned representations important for classification. Autoencoders provide nonlinear latent representations, while Product Coefficients provide geometric feature augmentations developed in earlier LiDAR work \cite{medina2025integrating}.

LiDAR point clouds provide dense three-dimensional representations of natural and built environments. Although modern LiDAR acquisitions typically include attributes such as intensity, return number, scan angle, and GPS time, this work focuses exclusively on the spatial coordinates $(x,y,z)$, making learned geometric representations particularly important for classification.

In this paper, we study Dynamical System Autoencoders (DSAE) in this setting. A DSAE repeatedly applies a learned transformation before extracting the hidden representation used by a downstream classifier. The number of encoder iterations $K$ is therefore a meaningful architectural parameter.

Our experiments reveal a sharp depth-dependent failure. Models trained with $K=1,\ldots,4$ retain strong classification performance. At $K=5$, however, the hidden representation becomes nearly constant. The collapse occurs for both xyz and xyz plus Product Coefficient inputs. At the same depth, Random Forest, kNN, and a majority-class Dummy classifier all attain the same macro F1 score.

We give a mathematical explanation of this behavior. Total hidden scatter decomposes into between-class and within-class scatter. Hence, when total hidden scatter approaches zero, between-class scatter must also approach zero. A nearly constant hidden representation therefore cannot preserve substantial class separation.

The paper makes three contributions. First, it evaluates DSAE depth across two LiDAR feature settings and two standard downstream classifiers. Second, it identifies a large-depth hidden-state collapse that persists across feature settings and random seeds. Third, it proves that the observed hidden variance bounds the class-separating scatter available to any downstream classifier.

\section{Related Work}

Dimensionality reduction and representation learning methods have become increasingly important in LiDAR point cloud classification due to the high dimensionality and geometric complexity of 3D spatial data. Classical approaches such as Principal Component Analysis (PCA), kernel approximations, and autoencoder-based latent representations have been widely studied for feature extraction and downstream classification tasks.

Product Coefficients originate from a dyadic measure representation framework developed in harmonic analysis and subsequently extended to multiscale representations for general dyadic data structures \cite{fefferman1991,bassu2020product,ness20xx}. More recently, they have been applied as geometric descriptors for LiDAR point cloud classification \cite{medina2025integrating, Medina2026PC}.
That study demonstrated that incorporating Product Coefficients alongside dimensionality reduction pipelines could improve classification accuracy relative to standard feature configurations. The framework also highlighted the potential role of hierarchical geometric descriptors in learning latent representations for point cloud data.

More recently, Dynamical System Autoencoders (DSAE) were introduced in \cite{he2024dsae} as an iterative latent representation framework motivated by dynamical systems formulations of autoencoder architectures. Unlike standard autoencoders, DSAE models apply iterative encoder compositions that may induce nontrivial latent geometric behavior across multiple iterations.

In this paper, we connect these two directions by investigating how iterative DSAE latent representations behave in LiDAR classification settings both with and without Product Coefficient feature augmentations. In particular, we study how encoder iteration depth influences downstream classification performance under different classifiers and latent feature configurations.

A continuation of this direction was further developed in \cite{medina2025integrating}, where Product Coefficients were investigated within expanded LiDAR classification pipelines and additional latent representation settings. That work explored how Product Coefficients interact with dimensionality reduction methods and downstream classifiers in more complex experimental configurations, further supporting their role as meaningful geometric descriptors for 3D point cloud data. The present paper extends these investigations into the setting of Dynamical System Autoencoders (DSAE), where the focus shifts from static latent reductions to iterative latent dynamical behavior across encoder compositions.

Random Forest constructs ensembles of decision trees from coordinate-based partitions of the feature space \cite{breiman2001random}, while kNN predicts from local neighborhoods \cite{cover1967nearest}. We use these classifiers as distinct probes of the learned hidden representation. Their convergence to the majority-class Dummy baseline at $K=5$ indicates that the failure occurs in the representation rather than in one particular classification rule.

\section{Methodology}

\subsection{Classical Autoencoders}

An autoencoder is a neural network architecture designed to learn a lower-dimensional representation of input data by reconstructing the input from a compressed latent representation. Given data points $x \in \mathbb{R}^n$, a classical autoencoder consists of two main maps: an encoder
\[
E:\mathbb{R}^n \to \mathbb{R}^m,
\]
where typically $m<n$, and a decoder
\[
D:\mathbb{R}^m \to \mathbb{R}^n.
\]

The encoder maps the input data to a latent variable
\[
z = E(x),
\]
while the decoder attempts to reconstruct the original input through
\[
\hat{x}=D(z)=D(E(x)).
\]

The encoder and decoder are trained by minimizing a reconstruction loss over a training data set $\{x_i\}_{i=1}^N$. A common choice is the mean squared error loss
\[
\mathcal{L}_{\mathrm{MSE}}
=
\frac{1}{N}\sum_{i=1}^N
\left\|x_i - D(E(x_i))\right\|_2^2.
\]

In this framework, the latent representation $z$ is expected to retain the most relevant information needed to reconstruct the input. Autoencoders are therefore commonly used for nonlinear dimensionality reduction, feature learning, denoising, and visualization.

In the context of this work, classical autoencoders provide a baseline method for reducing the dimension of feature vectors derived from 3D point cloud data. The learned latent variables can then be used as inputs for downstream classification tasks, allowing comparison with more structured approaches such as Dynamical System Autoencoders.

\subsection{Dynamical System Autoencoders}

Dynamical System Autoencoders, introduced by He et al. \cite{he2024dsae}, extend classical autoencoders by repeatedly applying a learned transformation. In the implementation used here, the repeated map acts on an extended state.

For an input feature vector $x_i\in\mathbb{R}^d$, let $h_i^{(k)}\in\mathbb{R}^r$ denote the hidden state after iteration $k$. We write
$$
u_i^{(k)}=\left(x_i^{(k)},h_i^{(k)}\right)\in\mathbb{R}^{d+r}.
$$
The initial state is
$$
u_i^{(0)}=(x_i,0),
$$
and the learned DSAE step is
$$
F_\theta:\mathbb{R}^{d+r}\longrightarrow\mathbb{R}^{d+r}.
$$
The iteration is
$$
u_i^{(k+1)}=F_\theta\left(u_i^{(k)}\right).
$$
After $K$ iterations,
$$
u_i^{(K)}=F_\theta^K\left(u_i^{(0)}\right)=\left(\widehat{x}_i^{(K)},h_i^{(K)}\right).
$$
The reconstruction loss is
$$
\mathcal{L}_{\mathrm{MSE}}=\frac{1}{N}\sum_{i=1}^N\left\|x_i-\widehat{x}_i^{(K)}\right\|_2^2,
$$
and the downstream classifiers are trained on $h_i^{(K)}$.

Each point on a depth curve in this paper comes from a separately trained DSAE architecture. The curves therefore describe an architecture-depth comparison. They do not represent intermediate states of one fixed learned map.

The iterative structure of DSAEs naturally introduces a dynamical systems perspective into latent representation learning.
In this work, we investigate how repeated latent compositions influence LiDAR-derived feature representations, including both spatial coordinates and Product Coefficient augmentations.

\subsection{Hidden-state scatter under DSAE depth}
\label{sec: hidden-state-scatter}

Let
$$
h_i^{(K)}\in\mathbb{R}^r
$$
be the hidden representation of sample $i$, and let $y_i$ be its class label. For each class $c$, define
$$
C_c=\{i:y_i=c\},\qquad N_c=|C_c|.
$$
The global hidden mean is
$$
\bar{h}^{(K)}=\frac{1}{N}\sum_{i=1}^N h_i^{(K)},
$$
and the class mean is
$$
\bar{h}_c^{(K)}=\frac{1}{N_c}\sum_{i\in C_c}h_i^{(K)}.
$$
Define the total hidden scatter
$$
T_K=\frac{1}{N}\sum_{i=1}^N\left\|h_i^{(K)}-\bar{h}^{(K)}\right\|^2,
$$
the between-class scatter
$$
B_K=\sum_c\frac{N_c}{N}\left\|\bar{h}_c^{(K)}-\bar{h}^{(K)}\right\|^2,
$$
and the within-class scatter
$$
W_K=\frac{1}{N}\sum_c\sum_{i\in C_c}\left\|h_i^{(K)}-\bar{h}_c^{(K)}\right\|^2.
$$
The scalar statistics reported in the experiments are
$$
\mu_K=\frac{1}{Nr}\sum_{i=1}^N\sum_{j=1}^r h_{ij}^{(K)}
$$
and
$$
\sigma_K^2=\frac{1}{Nr}\sum_{i=1}^N\sum_{j=1}^r\left(h_{ij}^{(K)}-\mu_K\right)^2.
$$
We report
$H_{\mathrm{mean}}=\mu_K, \quad H_{\mathrm{std}}=\sigma_K$.

\begin{proposition}For every depth $K$,
$$
0\leq B_K\leq T_K\leq r\sigma_K^2.
$$
\end{proposition}

\begin{proof}

For $i\in C_c$,
$$
h_i^{(K)}-\bar{h}^{(K)}=\left(h_i^{(K)}-\bar{h}_c^{(K)}\right)+\left(\bar{h}_c^{(K)}-\bar{h}^{(K)}\right).
$$
The cross term vanishes after summation over $C_c$, since
$$
\sum_{i\in C_c}\left(h_i^{(K)}-\bar{h}_c^{(K)}\right)=0.
$$
Summing over classes gives
$$
T_K=B_K+W_K.
$$
Since $W_K\geq 0$,
$$
0\leq B_K\leq T_K.
$$
Let $\mathbf{1}_r\in\mathbb{R}^r$ denote the vector whose entries are all equal to $1$. Since
$$
h_i^{(K)}-\mu_K\mathbf{1}_r=\left(h_i^{(K)}-\bar{h}^{(K)}\right)+\left(\bar{h}^{(K)}-\mu_K\mathbf{1}_r\right),
$$
we obtain
\begin{align*}
\frac{1}{N} & \sum_{i=1}^N\left\|h_i^{(K)} -\mu_K\mathbf{1}_r\right\|^2 =\frac{1}{N}\sum_{i=1}^N\left\|h_i^{(K)}-\bar{h}^{(K)}\right\|^2 \\ 
& +2\left\langle\frac{1}{N}\sum_{i=1}^N\left(h_i^{(K)}-  \bar{h}^{(K)}\right),\bar{h}^{(K)}-\mu_K\mathbf{1}_r\right\rangle \\
& + \left\|\bar{h}^{(K)}-\mu_K\mathbf{1}_r\right\|^2.
\end{align*}
The middle term vanishes because
$$
\frac{1}{N}\sum_{i=1}^N\left(h_i^{(K)}-\bar{h}^{(K)}\right)=0.
$$
Hence
$$
\frac{1}{N}\sum_{i=1}^N\left\|h_i^{(K)}-\mu_K\mathbf{1}_r\right\|^2=T_K+\left\|\bar{h}^{(K)}-\mu_K\mathbf{1}_r\right\|^2.
$$
By the definition of $\sigma_K^2$,
$$
\sigma_K^2=\frac{1}{Nr}\sum_{i=1}^N\left\|h_i^{(K)}-\mu_K\mathbf{1}_r\right\|^2.
$$
Therefore
$$
r\sigma_K^2=T_K+\left\|\bar{h}^{(K)}-\mu_K\mathbf{1}_r\right\|^2,
$$
and in particular
$$
T_K\leq r\sigma_K^2.
$$
\end{proof}

If $\sigma_K=0$, then
$$
h_i^{(K)}=\mu_K\mathbf{1}_r
$$
for every sample $i$. Hence any deterministic classifier whose input is only the hidden representation must assign the same label to every sample. At $K=5$, the representation is nearly, rather than exactly, collapsed. The agreement of the Random Forest and kNN macro F1 scores with the majority-class Dummy baseline is therefore consistent with the exact collapse limit and provides empirical evidence that neither classifier extracts class-separating information from the hidden representation beyond the majority-class baseline.

\section{Experimental Results}

\subsection{Experimental Setup}

We evaluate Dynamical System Autoencoders (DSAE) on LiDAR point cloud classification tasks using spatial coordinates $(x,y,z)$ and Product Coefficient feature augmentations.

The experiments are conducted on a terrestrial LiDAR scan of a forest tree collected in Washington State (Fig.~\ref{fig:lidar_scene}). The dataset was collected by Dr. Jonathan Batchelor while he was a member of the Remote Sensing and Geospatial Analysis Laboratory (RSGAL) under the direction of Dr. L. Monika Moskal at the University of Washington. The point cloud contains 798,452 labeled points belonging to three semantic classes: ground, trunk/branches, and canopy. Although the original LiDAR data include additional attributes such as intensity, return number, scan angle, GPS time, and RGB values, this study intentionally uses only the spatial coordinates $(x,y,z)$ in order to isolate the behavior of the learned latent representations and the effect of Product Coefficient feature augmentations.

\begin{figure}[!t]
\centering
\includegraphics[width=\columnwidth]{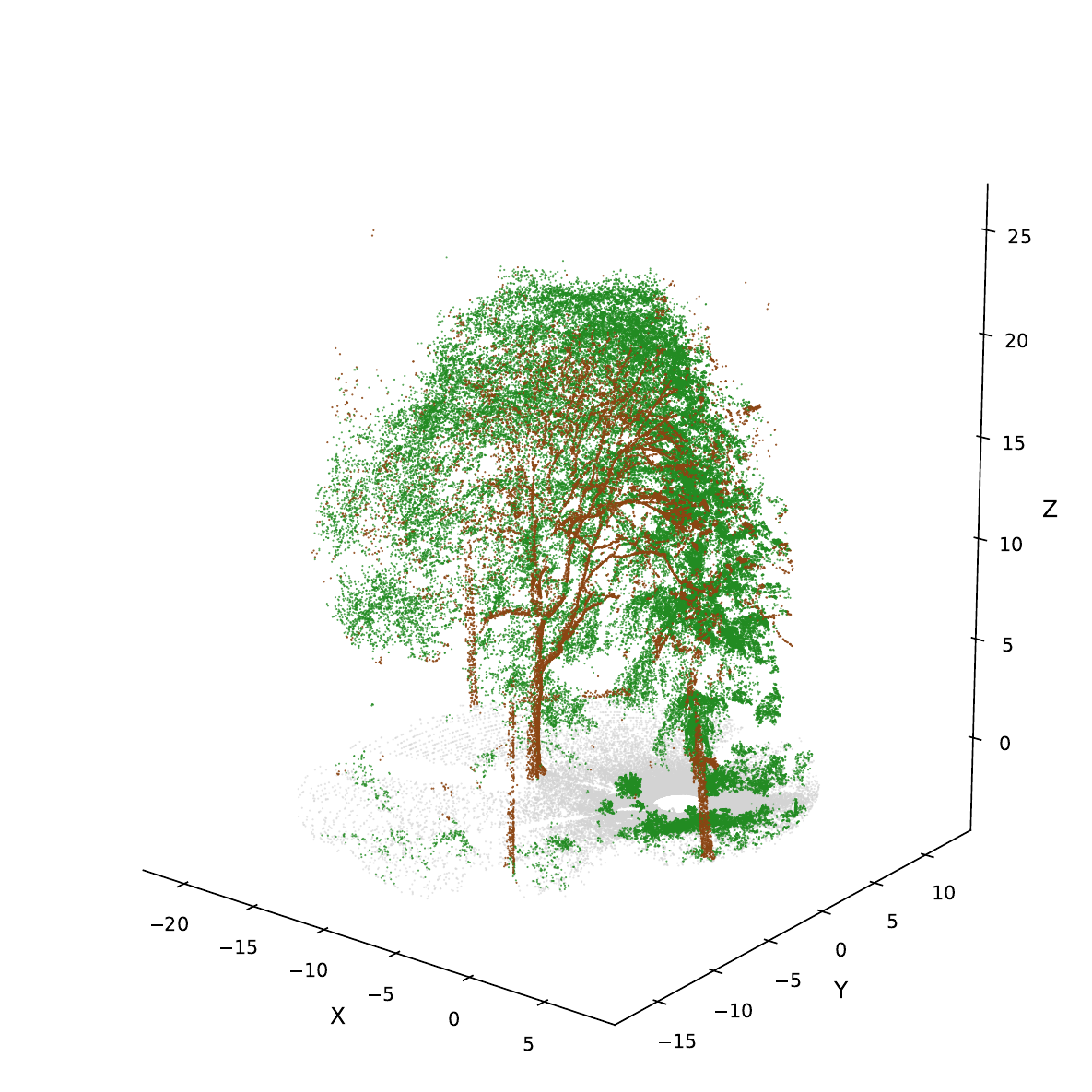}
\caption{Terrestrial LiDAR point cloud of a forest scene collected in Washington State. The point cloud contains 798,452 labeled points belonging to three semantic classes: ground, trunk/branches, and canopy, shown in gray, brown, and green, respectively. The dataset was collected by Dr.~Jonathan Batchelor while a member of the Remote Sensing and Geospatial Analysis Laboratory (RSGAL) under the direction of Dr.~L.~Monika Moskal at the University of Washington. Although the original LiDAR data include additional attributes such as intensity, return number, scan angle, GPS time, and RGB values, only the spatial coordinates $(x,y,z)$ are used in the experiments.}
\label{fig:lidar_scene}
\end{figure}

Each depth $K\in\{1,2,3,4,5\}$ was trained as a separate DSAE model from initialization. All runs used the same 80/20 stratified train-test split, with 638,761 training samples, 159,691 test samples, and 3 classes. Feature scaling was fitted on the training split and then applied unchanged to the test split. All DSAE models were trained for $25$ epochs using the Adam optimizer with a learning rate of $10^{-3}$ and a batch size of $256$. The hidden layer dimension was set to $3$ for the XYZ feature representation and $10$ for the combined XYZ + Product Coefficients (PCs) representation.

Random Forest used $100$ trees with the remaining hyperparameters set to the \texttt{scikit-learn} defaults. The kNN classifier used $k=5$ with the Euclidean distance metric (\texttt{metric='minkowski'} with $p=2$) and uniform weighting. The Dummy baseline used the most-frequent training class. Macro F1 was computed on the same test split for every classifier.

\begin{figure}[t]
\centering
\includegraphics[width=\linewidth]{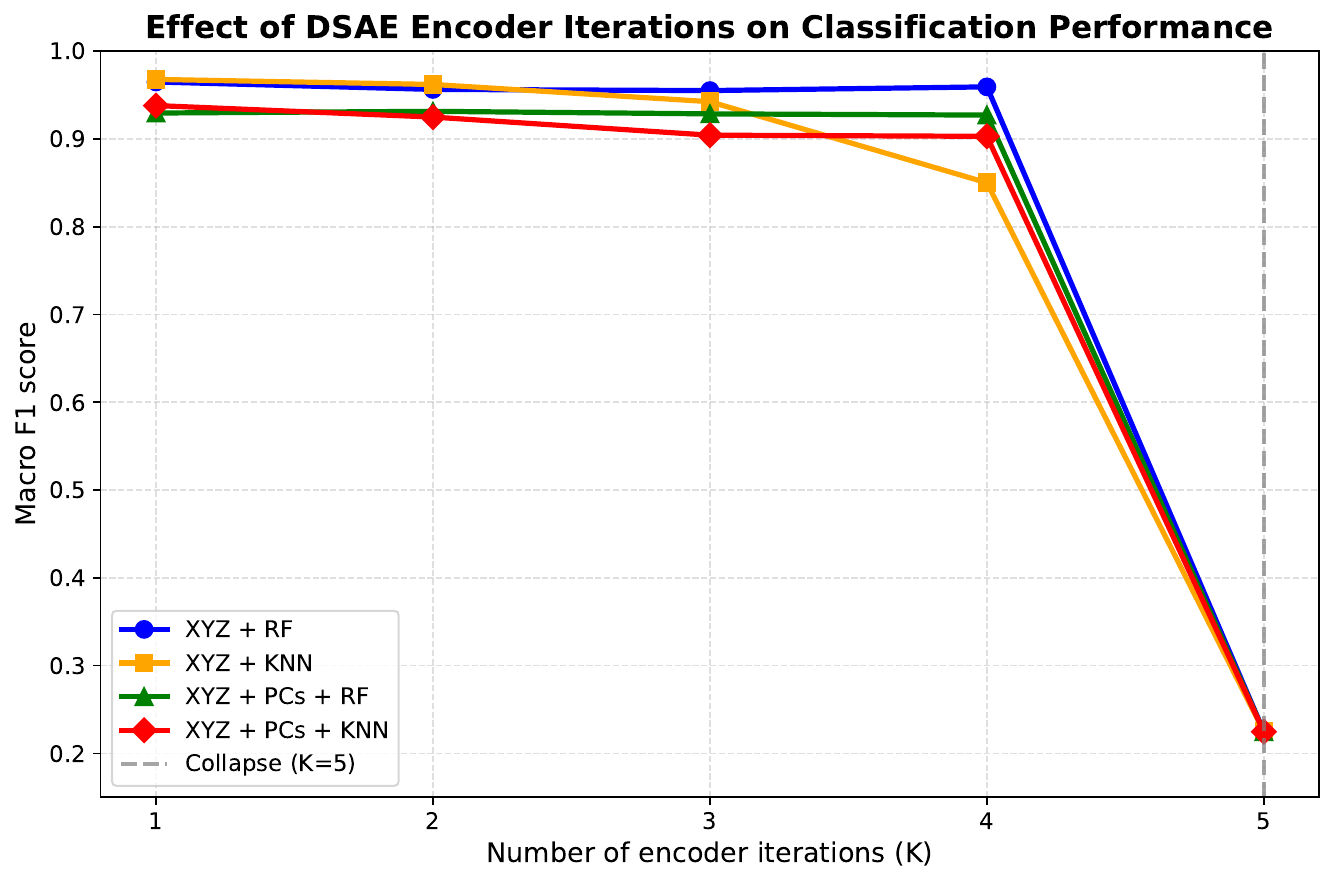}
\caption{Macro F1 as a function of DSAE encoder depth for xyz and xyz plus Product Coefficient inputs. Models at each depth were trained separately. Performance remains high through $K=4$, while all feature-classifier configurations collapse to the majority-class Dummy baseline at $K=5$.}
\label{fig:dsae_iterations_collapse}
\end{figure}

Hidden-state statistics were computed from the same test hidden representations used for downstream evaluation. We use the population standard deviation after flattening the hidden matrix, so $H_{\mathrm{mean}}=\mu_K$ and $H_{\mathrm{std}}=\sigma_K$, with $\mu_K$ and $\sigma_K$ defined in Section~\ref{sec: hidden-state-scatter}.

The experiments investigate how encoder iteration depth influences downstream classification performance. We evaluate the latent representations using both Random Forest (RF) and k-Nearest Neighbor (kNN) classifiers. Classification performance is measured using the macro F1 score.

We consider multiple encoder iteration depths and hidden dimensions to study the stability of iterative latent representations under different geometric feature settings.


\subsection{Depth-Dependent Classification Performance}

Figure~\ref{fig:dsae_iterations_collapse} shows macro F1 across encoder depths and feature settings. For $K=1,\ldots,4$, the learned representations retain strong classification performance. At $K=5$, every configuration collapses to macro F1 $0.224688$.

Because each depth was trained as a separate DSAE model, this is an architecture-depth result rather than a trajectory result for one fixed learned map. The sharp transition at $K=5$ shows that increasing DSAE depth can change the learned representation qualitatively rather than merely reducing classification accuracy gradually.

\subsection{$K=5$ Hidden-State Collapse}

To test whether the large-depth degradation is caused by collapse of the learned representation, we computed the hidden-state statistics and macro F1 scores at $K=5$. The results are shown in Table~\ref{tab:k5-collapse}. In both feature settings, the hidden-state standard deviation is on the order of $10^{-5}$. At the same depth, Random Forest, kNN, and the majority-class Dummy classifier all attain macro F1 $0.224688$.

This indicates that the $K=5$ hidden representation is nearly constant. By Proposition~1,
\[
B_5\leq T_5\leq rH_{\mathrm{std}}^2.
\]
Thus the between-class hidden scatter is also numerically negligible. The failure therefore occurs before classification. The downstream classifiers receive a hidden representation with essentially no usable class-separating variation.

For the $K=5$ latent representation, direct computation gives
\[
T_5 = 7.74\times 10^{-13},
\qquad
B_5 = 7.74\times 10^{-13}
\]
for xyz, and
\[
T_5 = 2.00\times 10^{-12},
\qquad
B_5 = 2.00\times 10^{-12}
\]
for xyz plus Product Coefficients. Both quantities are numerically close to zero.

For the xyz $K=5$ architecture, the collapse was also reproduced across five random seeds. RF macro F1 was $0.224688$ for every seed, while
\[
H_{\mathrm{std}}=(1.33\pm0.38)\times10^{-5}.
\]
The reconstruction loss varied across runs, but the collapsed hidden representation and downstream failure remained unchanged.

\begin{table*}[!t]
\centering
\small
\setlength{\tabcolsep}{7pt}
\caption{Hidden-state collapse at $K=5$. For both feature settings, the hidden-state standard deviation is on the order of $10^{-5}$, and all classifiers attain the majority-class Dummy macro F1.}
\label{tab:k5-collapse}
\begin{tabular}{lccccc}
\hline
Features
& $H_{\mathrm{mean}}$
& $H_{\mathrm{std}}$
& RF macro F1
& kNN macro F1
& Dummy macro F1 \\
\hline
xyz
& 0.999878
& $1.08\times10^{-5}$
& 0.224688
& 0.224688
& 0.224688 \\
xyz + PCs
& 0.999752
& $2.95\times10^{-5}$
& 0.224688
& 0.224688
& 0.224688 \\
\hline
\end{tabular}
\end{table*}

\subsection{Role of Product Coefficients}

Product Coefficients provide a second geometric feature setting for testing whether the collapse is specific to xyz input. In the present experiments, they do not prevent the $K=5$ collapse. The xyz plus Product Coefficient representation also becomes nearly constant, and RF, kNN, and Dummy attain the same macro F1.

Before collapse, Product Coefficients do not consistently improve macro F1 over xyz in this DSAE pipeline. We therefore do not claim a classification benefit from Product Coefficients in the present setting. Their role here is to show that the large-depth collapse is not confined to raw spatial coordinates. This result is specific to the current DSAE architecture and does not contradict earlier Product Coefficient gains obtained in different classification pipelines.

\section{Discussion}

The experiments identify a depth-dependent failure of the learned representation. At $K=5$, both input feature settings produce hidden states with standard deviation on the order of $10^{-5}$. The scatter bound
$$
B_K\leq T_K\leq rH_{\mathrm{std}}^2
$$
shows that such a representation cannot retain substantial between-class separation.

The convergence of Random Forest, kNN, and Dummy macro F1 scores gives classifier-independent evidence for this interpretation. Random Forest and kNN use different decision rules, yet both attain the majority-class baseline macro F1 after the hidden representation collapses. The failure is therefore best understood as representation collapse rather than classifier failure.

The present experiment does not prove why optimization produces collapse at $K=5$. It proves what the observed collapse implies for class separation, and it demonstrates that the same failure occurs across two feature settings. The depth curves also compare separately trained architectures. A fixed-map experiment evaluating intermediate states of one trained DSAE remains future work.

\section{Conclusion}

We studied the effect of DSAE encoder depth on LiDAR point-cloud classification using xyz coordinates and Product Coefficient feature augmentations. Models trained with $K=1,\ldots,4$ retain strong macro F1 performance. At $K=5$, both feature settings undergo hidden-state collapse.

The collapsed representations have standard deviation on the order of $10^{-5}$, and Random Forest, kNN, and a majority-class Dummy classifier all attain macro F1 $0.224688$. We proved that between-class scatter is bounded by total hidden scatter, which is itself controlled by the reported hidden-state variance. Thus the observed collapse removes the class-separating structure required for downstream classification.

Product Coefficients do not prevent collapse or consistently improve pre-collapse macro F1 in the present DSAE setting. Future work will study the optimization mechanism causing collapse, fixed-map intermediate states, explicit regularization of hidden scatter, and geometry-aware reconstruction losses.

\section*{Acknowledgments}

The authors thank Dr. Randy Paffenroth and Dr. Shiquan He of Worcester Polytechnic Institute for discussions concerning Dynamical System Autoencoders and learned latent representations. The authors also thank Dr. Shiquan He for providing access to the DSAE implementation used in the experiments.

\bibliographystyle{IEEEtran}
\bibliography{references-may26}

\end{document}